\documentclass{article}
\usepackage[preprint,nonatbib]{neurips_2020}
\usepackage[numbers]{natbib}
\usepackage[utf8]{inputenc} 
\usepackage[T1]{fontenc}    
\usepackage{hyperref}       
\usepackage{url}            
\usepackage{booktabs}       
\usepackage{amsfonts}       
\usepackage{nicefrac}       
\usepackage{microtype}      

\usepackage{amsthm}
\usepackage{graphicx} 
\usepackage[ruled,linesnumbered,resetcount,vlined,noend]{algorithm2e}
\usepackage[small]{caption}
\usepackage{subcaption}
\makeatletter

\newcommand{\nosemic}{\renewcommand{\@endalgocfline}{\relax}}
\newcommand{\dosemic}{\renewcommand{\@endalgocfline}{\algocf@endline}}
\let\oldnl\nl
\newcommand{\nonl}{\renewcommand{\nl}{\let\nl\oldnl}}

\SetKwFor{Repeat}{repeat}{:}{endw}
\SetKwFor{Upon}{upon}{:}{endw}
\makeatother
\usepackage{amsmath,amsfonts,mathrsfs,amssymb,bm}
\usepackage{todonotes}
\usepackage{nicefrac}
\usepackage{mathtools}
\usepackage{paralist}
\usepackage{enumitem}

\def\lev{\textit{lev}\,}
\def\succ{\textit{succ}\,}
\def\root{\textit{roots}\,}
\def\leaves{\textit{leaves}\,}

\def\chain{$\textrm{PT-DL}_C$}
\def\tree{$\textrm{PT-DL}_T$}
\def\star{$\textrm{PT-DL}_S$}

\def\thmspace{0.2em}
\newtheorem{theorem}{\hspace{\thmspace}{\bf Theorem}\!}

\newtheorem{definition}{\hspace{\thmspace}{\bf Definition}\!}

 \newcommand{\cD}{\mathcal{D}}

 \newcommand{\cL}{\mathcal{L}}
\newcommand{\cM}{\mathcal{M}} \newcommand{\cN}{\mathcal{N}}
 
\newcommand{\cR}{\mathcal{R}} 
 
 \newcommand{\cX}{\mathcal{X}}
\newcommand{\cY}{\mathcal{Y}}

\newcommand{\EE}{\mathbb{E}} \newcommand{\RR}{\mathbb{R}}

\newlength\myindent
\title{A Privacy-Preserving and Trustable Multi-agent Learning Framework}

\author{%
  Anudit Nagar\\
  {\small Bennett University}\\
  \texttt{anudit@bennett.edu.in}
  \And
  Cuong Tran\\
  {\small Syracuse University}\\
  \texttt{cutran@syr.edu}\\
  \And
  Ferdinando Fioretto \\
  {\small Syracuse University} \\
  \texttt{ffiorett@syr.edu}\\
}

\begin{document}

\maketitle

\begin{abstract}
Distributed multi-agent learning enables agents to cooperatively train a model without requiring to share their datasets. While this setting ensures some level of privacy, it has been shown that, even when data is not directly shared, the training process is vulnerable to privacy attacks including data reconstruction and model inversion attacks. Additionally, malicious agents that train on inverted labels or random data, may arbitrarily weaken the accuracy of the global model. This paper addresses these challenges and presents Privacy-preserving and trustable Distributed Learning (PT-DL), a fully decentralized framework that relies on Differential Privacy to guarantee strong privacy protection of the agents data, and Ethereum smart contracts to ensure trustability. The paper shows that PT-DL is resilient up to a 50\% collusion attack, with high probability, in a malicious trust model and the experimental evaluation illustrates the benefits of the proposed model as a privacy-preserving and trustable distributed multi-agent learning system on several classification tasks.
\end{abstract}
\section{Introduction}
\label{sec:introduction}

Agencies are increasingly leveraging distributed data shared across 
organizations to augment their AI-powered services. Examples include 
transportation services sharing location-based data to improve 
their on-demand capabilities and hospitals sharing patient data to 
prevent epidemic outbreaks. The proliferation of these applications leads to 
a transition from proprietary data acquisition and processing to 
distributed data ecosystems, where agents learn and make decisions 
using data owned by different organizations.

While the cooperative use of rich dataset may enhance the accuracy and
robustness of the individual agents' models, many of these
datasets contain sensitive information whose use is strictly regulated 
by local policies.  For example, medical data is regulated by HIPAA regulations \cite{hipaa} in the US, and the E.U.~General Data Protection Regulation (GDPR) requires the consumer data privacy to be protected \cite{geyer2017differentially}. 

Distributed multi-agent learning enables agents to 
cooperatively train a learning model without requiring them to share 
their dataset. Typical multi-agent learning frameworks, including 
federated learning \cite{mcmahan2017communicationefficient}, allow individual agents to train their
local models on their own datasets and share model parameters with a 
centralized agent that aggregates the received parameters and
sends them back to the agents. This simple, yet 
effective procedure, repeats for several iterations and allows the participating agents to learn a global model without accessing data of other agents.

However, standard multi-agent learning algorithms are susceptible to 
privacy and trustability issues.  Indeed, using solely the model
parameters, model inversion attacks can  be performed to gain insight
about the input data used during the training 
\cite{geiping2020inverting,fang2019local}, and
membership inference attacks have been shown successful to
determine if an individual data was used or not during model training \cite{shokri:17}. These privacy-invasive attacks can be  mitigated with the adoption of \emph{Differential Privacy}, a formal  privacy model that
computes and bounds the privacy loss associated to  the participation
of an individual into a computation.  However, differential privacy
does not protect against malicious agents  that may not comply with
the protocol--e.g., by altering their data  labels, performing what
known as \emph{model inversion attacks} \cite{10.1145/2810103.2813677}, or training over
random data.  These behaviors can strongly penalize the accuracy of the
resulting global model. Defending against such attacks requires some form of \emph{trustability} and exclude the flagged contributions from the global model updates. The decentralized setting of multi-agent learning renders this task particularly challenging.

To address these issues, this paper proposed the use a decentralized
computational environment, such as blockchains, that enables
differentially private multi-agent training, guaranteeing both privacy and trustability.  The resulting framework, called \emph{trustable and Private Distributed Learning}
(PT-DL) relies on the Ethereum blockchain, that combines an immutable
data storage with a Turing-complete computational environment 
\cite{wood2014ethereum} and guarantees the correctness of the programs executed over the blockchain. 

The privacy requirement is enforced by using a clipping approach on
the model parameters and the privacy analysis relies on composition methods \cite{dwork:14} and the moment accountant for the
Sampled Gaussian Mechanism \cite{abadi:16,mironov2019rnyi}. 
trustability is achieved directly through the use of the computation ran on the immutable blockchain combined with
a decentralized verification procedure that validates the genuineness of the agent contributions to the model.
The paper shows that the trustable scheme implemented is robust, with high probability, up to the case where half of the agents may collude.

In addition to these properties, PT-DL requires no central trusted server
and allows agents to exploit several network communication topologies to reduce the negative impact of frequent model aggregations as well as lowering the communication costs. 
The empirical analysis illustrates the benefits of the proposed model on classification and medical diagnosis tasks, based on a COVID-19 dataset, in terms of accuracy, communication cost, and resiliency to biased data distributions and malicious updates. Additionally, it reports an excellent trade-off between accuracy and privacy and shows that PT-DL may represent a promising step towards a practical tool for privacy-preserving and trustable multi-agent learning.

This work is an extended version of \cite{Nagar:AAMAS21}.

\section{Related Work}
\label{sec:related_work}

Multi-agent learning has been studied extensively, especially in the context of optimization \cite{balcan2012distributed,fercoq2014fast,shamir2014distributed,Fioretto:JAIR-18} and communication efficiency 
\cite{shamir2014distributed,zhang2013information,yang2013trading,zhang2018communication,ma2015adding,Fioretto:AAAI-16}, e.g., by compressing the information that is passed among the agents during training. 

In the context of learning non-convex functions, as in deep learning tasks, from decentralized data, \emph{Federated learning} \cite{konevcny2015federated,mcmahan2016communication} is considered the de-facto standard framework. 
A cornerstone in federated learning algorithms is \textsl{Federated Averaging} (FedAvg), a variant of distributed stochastic gradient descent (SGD). At each iteration, a selected subset of FedAvg agents locally and independently take a gradient descent step on the current model parameters using their local data, and the server then computes a weighted average of the resulting parameters. 
There are several contributions that apply differential privacy in the context of multi-agent learning. \citet{cheng2018leasgd,cheng_19} 
focus on a differentially private decentralized learning system, but their analysis is applied to strongly convex problems only. 
\citet{bellet2018personalized,Bellet2017FastAD} 
proposes a differentially private asynchronous decentralized learning systems for convex problems and uses the Laplace mechanism over a block coordinate descent algorithm to achieve privacy.
\citet{Fioretto:AAMAS-19} develops a federated data sharing protocol 
which guarantees differential privacy. 

A different line of work applies the Alternating Direction Method of Multipliers to solve a distributed optimization problem under privacy constraints \cite{8736857,7563366,zhang2018improving}. These, however, are not generally adopted for training deep learning models in a distributed system of agents. 
An important contribution, by \citet{geyer2017differentially}, develops a differentially private federated learning system which builds on FedAvg. This paper uses it as a baseline. 

Although these systems guarantee privacy, the reported experiments show that the induced accuracy is not always satisfactory and, in addition, consider privacy in isolation, ignoring trustability. 
Finally, there have also been several proposals of using blockchain to support the training of deep neural networks \cite{weng2019deepchain} and federated learning \cite{kim2018device} for incentive purpose. 

In contrast to the work above, this paper develops a distributed 
multi-agent learning framework that ensures both privacy, 
through the use of differentially private model updates, and trustability, coordinating computations on a blockchain. Finally,
it proposes a simple, yet effective, communication model 
that may reduce the number of communication rounds required during model training when compared to standard federated learning approaches.

\section{Preliminaries}

\subsection{Problem Settings and Goals}
\label{sec:problem_definition}
The paper considers a collection of $K$ agents, each holding a dataset 
$D_a$ ($a \in [K]$) consisting of $n_a$ individual data points $(X_i, Y_i)$, 
with $i \!\in\! [n_a]$ drawn from an unknown distribution. 
Therein, $X_i \!\in\! \mathcal{X}$ is a feature vector and $Y_i \!\in\! 
\mathcal{Y}$ is a label. 
The paper assumes that an individual data is not repeated across 
datasets, and, thus, the agents' datasets are disjoint. Finally, it 
denotes with $D = \cup_{a \in [K]} D_a$ as the union of the agents' datasets and with 
$n = \sum_{a \in [K]} n_a$ the size of $D$.
For example, consider the case of a classifier that needs to predict 
the presence of a particular disease. The training example features 
$X_i$ may describe the individual's X-Ray image, gender, age, and 
demographics, and $Y_i$ represents whether the individual 
has the disease.
Each agent may hold a small dataset with samples representative of a specific demographics and, thus, the learning task may benefit from  using data samples of other agents.
Therefore, the goal is to learn a \emph{global} classifier 
$\cM_\theta: \cX \to \cY$, where $\theta$ is a real valued vector 
describing the model parameters. The model quality is measured in terms
of a non-negative, and assumed differentiable, \emph{loss function}
$\cL : \cY \times \cY \to \mathbb{R}_+$, and the problem is that of 
minimizing the empirical risk function:
\begin{equation}
\label{eq:ERM}
\min_{\theta} J(\cM_\theta, D) = 
\frac{1}{K} \sum_{a \in [K]} J_a(\cM^a_\theta, D_a),
\end{equation}
where $\cM_\theta^a$ is a local classifier associated to agent $a \in [K]$ and $J_a$ is its empirical risk function, defined as
\begin{equation}
\label{eq:ERM_agent}
J_a(\cM^a_\theta, D_a) = \frac{1}{n_a}
  \sum_{(X_i, Y_i) \in D_a} \cL\left(\cM^a_\theta(X_i), Y_i \right).
\end{equation}
The paper focuses on instances of the problem \eqref{eq:ERM} where 
the functions $J_a$ are non-convex, as in deep learning tasks.

The setting \emph{does not} assume that the data contributed by each of the $K$ agents are \emph{i.i.d.} or \emph{balanced}, that is, different agents may have data that is not representative of the whole population. 

As discussed in the introduction, there are two practical aspects 
that make the development of systems that solve problem \eqref{eq:ERM} challenging: privacy and trustability. The following sections review the main concepts adopted to address these challenges.

\subsection{Blockchains} \label{sub:blockchain}

\emph{Blockchains} are decentralized general transaction ledgers 
that allow participants to create \emph{unchangeable} records, each
time-stamped and linked to the previous one, making each operation 
verifiable and auditable. A protocol executed on a blockchain is 
referred to as \emph{smart contract} 
and this paper adopts the \emph{Ethereum} protocol, a Turing-complete 
computational environment \cite{wood2014ethereum}. 
The Ethereum protocol ensures that smart contracts are executed
correctly, thus, agents can trust that any data sent to the
blockchain will not be corrupted and that smart contracts logic will 
be executed as intended. With these guarantees, blockchains are appropriate to run distributed learning trustability procedures. 

While this environment guarantees that the data stored on the blockchain is immutable, it does not guarantee data privacy.

\subsection{Differential Privacy}
\label{sub:differential_privacy}

Differential privacy (DP) \cite{dwork:06} is a strong privacy notion
used to quantify and bound the privacy loss of an individual
participation to a computation.  
\emph{The privacy goal of this work is to guarantee that the output of 
the learning model does not differ much when a single individual is 
added or removed to the dataset}, limiting the amount of information 
that the model reveals about any individual.  

The action of changing a single attribute from a dataset $D$,
resulting in a new dataset $D'$, defines the notion of \emph{dataset
adjacency}. Two dataset $D$ and $D' \in \mathcal{X}^n$ are said
adjacent, denoted $D \sim D'$, if they differ in at most the addition 
or removal of a single entry (e.g., in one individual's participation).

\begin{definition}[Differential Privacy]
  \label{dp-def}
  A randomized mechanism $\mathcal{M} : \mathcal{X}^n \to \mathcal{R}$ 
  with domain $\mathcal{X}^n$ and range $\mathcal{R}$ is $(\epsilon, 
  \delta)$-differentially private if, for any two adjacent datasets  
  $D, D' \in \mathcal{X}^n$, and any subset of 
  output responses $R \subseteq \mathcal{R}$:
  {
  \[
      \Pr[\mathcal{M}(D) \in R ] \leq  e^{\epsilon} 
      \Pr[\mathcal{M}(D') \in R ] + \delta.
  \]
  }
\end{definition}
\noindent When $\delta=0$ the algorithm is said to satisfy \emph{pure}
differential privacy. 
Parameter $\epsilon > 0$ describes the \emph{privacy loss} of the algorithm, 
with values close to $0$ denoting strong privacy, while parameter 
$\delta \in [0,1]$ captures the probability of failure of the algorithm to 
satisfy $\epsilon$-differential privacy. The global sensitivity $\Delta_f$ of a real-valued 
function $f: \mathcal{X}^n \to \mathbb{R}^k$ is defined as the maximum amount 
by which $f$ changes  in two adjacent inputs $D$ and $D'$:
\(
  \Delta_f = \max_{D \sim D'} \| f(D) - f(D') \|_2.
\)
In particular, the Gaussian mechanism, defined by
\[
    \mathcal{M}(D) = f(D) + \mathcal{N}(0, \Delta_f^2 \, \sigma^2), 
\]
\noindent where $\mathcal{N}(0, \Delta_f\, \sigma^2)$ is 
the Gaussian distribution with $0$ mean and standard deviation 
$\Delta_f\, \sigma^2$, satisfies $(\epsilon, \delta)$-differential privacy for 
$\delta > \frac{4}{5} \exp(-(\sigma\epsilon)^2 / 2)$ 
and $\epsilon < 1$ \cite{dwork:14}. 

Differential privacy also satisfies several important properties.
In particular, \emph{sequential composition} theorems allow to reason 
about the privacy loss resulting by the repeated application of a 
function to a dataset \cite{dwork:14,kairouz2015composition}, 
\emph{parallel composition} \cite{dwork:14} ensures that the application 
of an $(\epsilon, \delta)$-differentially private mechanism to disjoint dataset does not induce additional privacy loss, and 
\emph{post-processing immunity} \cite{dwork:14} ensures that differential privacy is preserved even if the output of a mechanism is post-processed throughout an arbitrarily data-independent transformation.




Although advanced composition theorems allow a tight privacy loss 
analysis, which scales sublinearly in the number of applications of 
the DP mechanisms, in iterative algorithms, like those used in this 
paper, this process quickly leads to large privacy losses. To 
address the shortcomings, this paper adopts the \emph{R\'{e}nyi 
Differential Privacy} (RDP) model by \citet{Mironov_2017}. 
RDP shares many important properties with the standard definition of 
differential privacy, including parallel composition and 
post-processing immunity, while additionally allowing a tighter 
privacy loss analysis for iterative processes (see Section \ref{sec:privacy_analysis} for details).


\begin{figure}[!t]
\centering
\includegraphics[width=0.95\columnwidth]{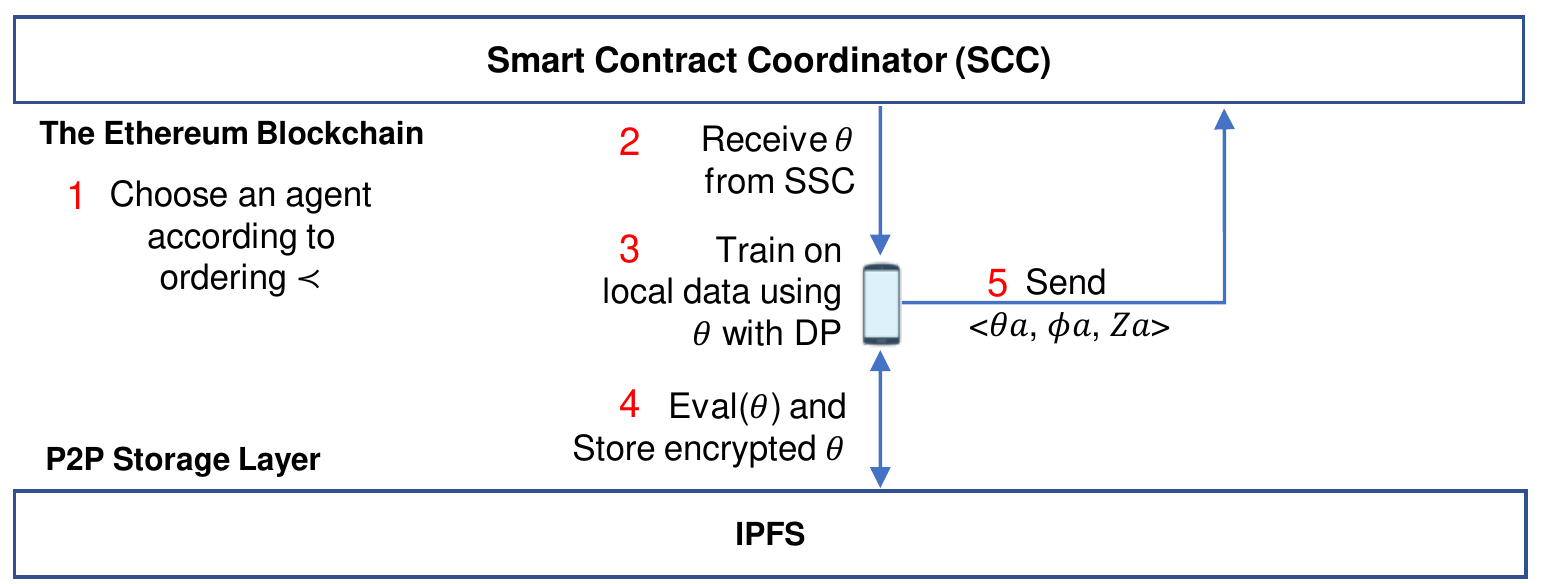}
\caption{Flow diagram of the $\textrm{PT-DL}$ Framework}
\label{fig:framework}
\end{figure}

\section{The PT-DL Framework} %
\label{sec:block_chain_distributed_learning}

\emph{Private and trustable Distributed Learning} (PT-DL) is a fully distributed learning framework that ensures privacy and trustability while keeping the network bandwidth low. 
The framework is schematically shown in Figure \ref{fig:framework} and its main components are:
\begin{itemize}[leftmargin=*, parsep=2pt, itemsep=0pt, topsep=0pt]
  \item A \emph{Smart Contract Coordinator (SCC)}: 
  It is a program executed on the blockchain that orchestrates the 
  interaction among PT-DL agents to ensure the correct data exchange 
  aimed at training a global model. The SCC operates in \emph{rounds}.
  At each round a set of agents is invoked according to a predefined 
  ordering and their responses are aggregated. 

  \item A provably private \emph{PT-DL agent} training procedure:  
  At each round, the invoked agents use the parameters obtained by 
  the SCC to train a model over their dataset. Each training step 
  is ensured to guarantee $(\epsilon, \delta)$-differential privacy.

  \item \emph{trustability}: Prior being able to submit a model 
  update, a PT-DL agent is required to invoke a verification step 
  that ensures its trustworthiness.
\end{itemize}

\section{Privacy-Preserving and trustable Distributed Learning} %
\label{sec:block_chain_distributed_learning}

\emph{Private and trustable Distributed Learning} (PT-DL) is a fully distributed learning framework that ensures privacy and trustability while keeping the network bandwidth low. 
The framework is schematically illustrated in Figure \ref{fig:framework} and its main components are:
\begin{itemize}[leftmargin=*, parsep=2pt, itemsep=0pt, topsep=0pt]
  \item A \emph{Smart Contract Coordinator (SCC)}: 
  It is a program executed on the blockchain that orchestrates the 
  interaction among PT-DL agents to ensure the correct data exchange 
  aimed at training a global model. The SCC operates in \emph{rounds}.
  At each round a set of agents is invoked according to a predefined 
  ordering (outlined in Section \ref{sec:topologies}) and their responses are aggregated. 

  \item A provably private \emph{PT-DL agent} training procedure:  
  At each round, the invoked agents use the parameters obtained by 
  the SCC to train a model over their dataset. Each training step  is ensured to guarantee $(\epsilon, \delta)$-differential privacy 
  (see Section \ref{sec:privacy_analysis}).

  \item \emph{trustability}: Prior being able to submit a model update, a PT-DL agent is required to invoke a verification step that ensures its trustworthiness.
\end{itemize}

Prior to discussing these components in detail, this section 
introduces the following notation. The PT-DL protocol processes agents according to a given ordering $\prec$, which, in turn, allows the procedure to exploit an underlying topology to perform parallel updates or to reduce the number of aggregation operations, as described in Section \ref{sec:topologies}.
An agent $a$ is a \emph{successor} of agent $a'$ if $a' \prec a$ and 
$\lev(a) = \lev(a')+1$, where
$\lev(a)$ denotes the level of $a$ in the ordering $\prec$.
The set $\succ(a) = \{a' \prec a \mid \forall a' \in[K] \land \lev(a) = \lev(a')+1\}$
describes the set of agents that succeed agents $a$ in the ordering $\prec$.
Finally, $\root$ denotes the set of all nodes with no predecessors and $\leaves$ 
the set of all nodes with no successors.

In the following, to simplify exposition, the discussion treats interactions between the SCC and PT-DL agents as direct data exchange. However, as sketched in Figure \ref{fig:framework}, the SCC coordinates a peer-to-peer (P2P) interaction between agents, that directly exchanges data using a secure protocol, as introduced below.

\subsection{PT-DL Smart Contract Coordinator} 
\label{sub:APDL_smart_contract}

The PT-DL SCC is described in Algorithm \ref{alg:scc}.
It takes as input the list of registered agents,
the agent ordering $\prec$, and threshold values $\kappa_1, \kappa_2 > 0$ that control the acceptance of the model updates submitted by an agent. 
The algorithm starts by initializing the model parameters $\theta$ (line 1)
and, for each round, it coordinates the model updates for the distributed 
training. At each round, the SCC starts by sending the current model 
parameters $\theta$ to the collection of agents that are currently 
active in \root{} (lines 3, 5). Here, the word \emph{active} is used to 
emphasize that the algorithm is resilient to agents that may drop out 
from the process, and thus become inactive.
Each invoked agent $a$ responds with a triple, containing the trained 
model parameters $\theta_a$, an evaluation score $\phi_a \in [0,1]$, e.g., a scalar describing the agent model $\cM^a_{\theta_a}$ accuracy, and a Zero Knowledge Proof value $Z_a$, which is used by the SCC to verify that the information sent by $a$ is indeed computed by $a$ (see Supplemental material for additional information). 
Next, for every active agent $a$ in $A$, the SCC invokes a distributed
evaluation procedure whose goal is to determine whether the model update was genuine or malignant. The evaluation involves a sub-sample $A_S$ of agents $a'$ that is asked to test model $\cM^{a'}_{\theta_a}$ over their test data (a subset of $D_{a'}$) and to report the associated accuracy metrics $\phi_{a'}$ (line 8). The collection of these metrics is denoted $\bm{\phi}_{A_S}$. 
A model update is retained genuine if the median accuracy score 
$\bar{\bm{\phi}}_{A_S}$ does not degrade substantially the current (global) model accuracy $\phi_A$ (first condition of line 9) and if it is not too far from the score $\phi_a$ reported by agent
$a$ (second condition of line 9).
SCC instructs the successors agents of $a$ to continue
the model training starting from parameter $\theta_a$, if the model update is accepted (line 10), or from $\theta$ if it is rejected (line 12).
Finally, agent $a$ is removed from the set $A$ of the agents to invoke in the current round (line 13). 
When all agents have been invoked, all valid models are aggregated via averaging (line 14), and the procedure repeats.

\begin{algorithm}[!t]
    \SetAlgoLined
    \KwIn{$\{a_i\}_{i\in [K]}$: The set of participating agents;\\
          \hspace{28pt}$\prec$: The agent ordering;\\
          \hspace{28pt}$\kappa_1, \kappa_2 \in [0,1]$: Accuracy thresholds}
    \KwOut{Global trained model $\theta$}

    $\theta \gets \bm{0}^T$\\
    \For{round $r = 1, 2, \ldots $ }{
      $A \gets \{ (a, \theta) \mid a \in \root\}$\\
      \While{$A \neq \emptyset$}{
        \textbf{send} $\theta$ to all active $a \in A$\\
        \textbf{receive} $(\theta_a, \phi_a, Z_a)$ from all active 
        $a \in A$\\
        \ForAll{$(a, \theta) \in A$}
        {
          $\bm{\phi}_{A_S} \gets \{\textsc{eval}_{a'}(\theta_a) \mid a' \!\in\! A_S, \text{subsample of } A\}$\\
          \eIf{$\phi_A - {\bm{\bar{\phi}}}_{A_S} \leq \kappa_1$ 
                and 
               $|\bm{\bar{\phi}}_{A_S} - \phi_a| \leq \kappa_2$}
          {
            $A \gets A \cup \{(a', \theta_a) \mid a' \in \text{succ}(a)\}$
          }{
            $A \gets A \cup \{(a', \theta) \mid a' \in \text{succ}(a)\}$
          }
          $A \gets A \smallsetminus \{(a, \theta)\}$\\
        }
      }
      $\theta \gets \textsc{Aggregate}(\{\theta_a \mid a \in \text{leaves} 
      \land \text{accepted}(\theta_a)\})$
    }
    \caption{PT-DL Smart Contract Coordinator (SCC)}
    \label{alg:scc}
\end{algorithm}

While not explicitly described, notice that the set of participating agents can be dynamically updated during the procedure.

\subsection{PT-DL Training Agent}
\label{sub:the_agent_model}

The PT-DL agents interact with the SCC using an Ethereum client, which is used to send and receive updates from the blockchain. 
PT-DL agents use the \emph{InterPlanetary File System} (IPFS) to store and share their models. IPFS is an efficient, secure, and peer-to-peer protocol. Prior to sending their data to the IPFS, the agents encrypt it using an asymmetric encryption scheme with keys derived from the sender and receivers' Ethereum accounts. Since IPFS is a public system, encryption secures the (differentially private) model updates to external accesses.

Each PT-DL agent operates as described in Algorithm \ref{alg:agent}. Given 
its training data $D_a$, the model learning rate $\eta$, and 
a mini-batch size $b$, the agent starts registering with the SCC (line 1) and hence it listens for updates. 
The agent responds to any direct request of its model parameters $\theta_a$ by sharing them with the requesting agent using IPFS (lines 1 and 2).
Upon receiving the model parameters $\theta_a$ from some agent $a'$\footnote{
  While the coordination process is orchestrated via the SCC the model requests are peer-to-peer, and thus agents use direct data exchange to execute the PT-DL protocol.} 
the agent decrypts them  and uses them to initialize its model $\cM^a$ (line 6). 
The training processes computes the gradients over the loss function 
$\cL(\cM^a_{\theta_a}(X_i), Y_i)$ for each data sample $(X_i, Y_i)$ in a minibatch $B$ of $D_a$. The computation of these gradients is made differentially private by the introduction of Gaussian noise calibrated as shown in Definition~\ref{def:SGM} (lines 7 and 8). 
The concept relies on performing a \emph{differentially private Stochastic Gradient Descent} (DP-SDG) step \cite{abadi:16}. In a nutshell, DP-SDG computes the gradients for each data sample in a mini-batch $B$, 
clips their $L2$-norm, computes the average, and add noise to ensure 
privacy. The process is illustrated in line 8. Therein, 
\(
  \nabla^c x = 
  \nicefrac{\nabla x}{\max(1, \frac{\|\nabla x\|}{c})}
\)
denotes the gradient of a given scalar loss $x$ clipped in a $c$-ball
for $c > 0$, $\sigma > 0$ is the Gaussian standard deviation value, and $\bm{I}$ is the identity matrix. A discussion on the privacy analysis is provided in Section \ref{sec:privacy_analysis}. 
Finally, the agent evaluates the trained model over its test data (line 9) and sends the evaluation metrics, its encrypted model parameters $\theta_a$, and a zero knowledge proof $Z_a$ of $\theta_a$ to the SCC (lines 10 and 11).

\begin{algorithm}[!t]
  \caption{Agent $a$ Training Routine}
  \label{alg:agent}
  \KwIn{$D_a$: The agent training data,\\
        \hspace{28pt}$\eta$: The learning rate\\
        \hspace{28pt}$b$: The mini-batch size
        }
  Register $a$ with the PT-DL SCC \\
  \Repeat{}{
    \If{agent $a'$ requests $\theta_a$}{
      encrypt $\theta_a$ and share it with $a'$ on the IPFS
    }
    \Upon{receiving $\theta$ (as $\theta_a)$ from agent $a'$}{
      decrypt $\theta_a$ and initialize model $\cM^a_{\theta_a}$\\
      \For{mini-batch $B$ of size $b$ in $D_a$} {
        $\displaystyle 
          \theta_a \gets \theta_a - \eta \;\;\nicefrac{1}{b}\!\!\!\!\!\!
          \sum_{(X_i, Y_i)\in B}\!\!\!\!
          \nabla^c_{\theta_a} \left[\cL(\cM^a_{\theta_a}(X_i), Y_i)\right] + \xi$\\
          \nonl
          \hspace{20pt}with $\xi \sim \cN(0, \sigma^2 \bm{I})$ as defined in Def.~\ref{def:SGM}\\
      }
      $\phi_a \gets \textsc{Eval}(\theta_a)$ over test data\\
      Generate Zero knowledge proof $Z_a$\\
      \textbf{send} encrypted $(\theta_a, \phi_a, Z_a)$ on the IPFS\\
    }
  }
\end{algorithm}

\subsection{Trustability and Incentives}
\label{sub:incentives_and_trustability}

PT-DL agents are required to assess other agents' contribution at each 
round of the distributed training process. The agents, which agree on 
an evaluation metric to adopt, are encouraged to submit highly accurate
models. As reviewed in Section \ref{sub:APDL_smart_contract}, a model 
update is accepted if the agent evaluation of that model does not 
deviate too much from the median evaluation score of a subsample $A_S$ of other agents and if the median score is also not much worse than the evaluation of the current global model. 
Submitting a low quality model would result in a low median score, 
while submitting an incorrect quality assessment would result in a high deviation from the median quality analysis. Both cases would result in rejecting the model update. 
Agents associated to valid (malicious) models can be rewarded (made trustable) through the use of the zero knowledge proof. 
Storing the zero knowledge proof on-chain allows the progress made by an agent to be verified by other agents without revealing any data. The process can be used to disincentive agents publishing poor model updates. A formal analysis on the trustability guarantee is given in Section \ref{sec:privacy_analysis}.

\subsection{PT-DL Communication Topologies}
\label{sec:topologies}

\begin{figure}[!t]
\centering
\includegraphics[width=0.8\columnwidth]{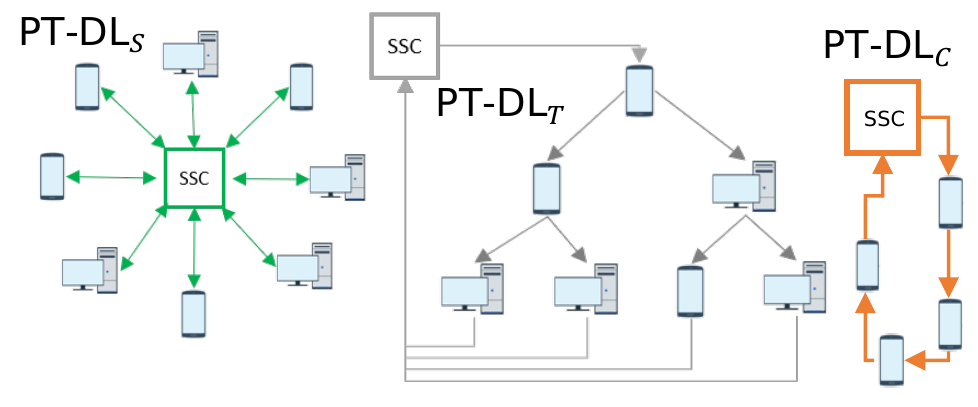} 
\caption{Conceptual Topological Diagram}
\label{fig-topos}
\end{figure}

In addition to privacy and trustability, a further aspect that 
challenges the development of systems that solve problem 
\eqref{eq:ERM} are the frequency of communication and model aggregations.
The former is a bottleneck in distributed training due to the large size of the parameters being exchanged.
The latter, as shown in the experimental analysis, may be deleterious to the model accuracy.
This section proposes several agent communication topologies 
that can ensure a contained communication cost and model aggregations.



The classical topology used in distributed learning is a star topology (see Figure \ref{fig-topos} (\star)), where, at each round all agents send their updates to the SCC, which, in turn, aggregates them. This topology allows the maximal level of parallelism, as all agents can perform local updates at the same time. 
However, the resulting aggregations may decrease the quality of the overall model as well as increase the network bandwidth utilization. 
Two further topologies adopted in this work are \tree{}, 
based on a tree decomposition, and \chain{}, that uses a chain, as illustrated conceptually in Figure \ref{fig-topos}. 
These topologies are encoded in the agent ordering $\prec$ given as input to Algorithm \ref{alg:scc},
The tree decomposition ensures that the aggregation is to be performed only by the leaf agents, while the chain topology requires no aggregations. 

While the number of aggregation operations and the network bandwidth reduces as the diameter of the graph induced by the agent communication topology increases, the algorithm runtime increases, since the number of parallel operations that can be performed decreases (logarithmically for the tree topology and linearly for the chain one) when compared to the start topology. This trade-off is the subject of study in Section \ref{sec:experimental_analysis}. 
To quantify this degree of parallelism the paper introduces the following terminology.
\begin{definition}[ACR]
\label{def2}
An Asynchronous Communication Round (ACR) is defined as the total number of parallel data exchanges amongst the agents and the SCC to complete an algorithmic round. 
\end{definition}
Notice that the chain topology requires $K$ exchanges, the tree topology requires $\log(K)$ exchanges, while the star topology requires a single parallel data exchange to complete a round. 

\section{Trustability and Privacy Analysis}
\label{sec:privacy_analysis}

To detect malicious updates, PT-DL requires each agent $a$ to report its accuracy metrics $\phi_a$ and compare it against the median metrics $\bm{\bar{\phi}}_{A_S}$ assessed by other agents in $A_S \subseteq [K]$. 
The trustability analysis relies on bounding the distance from the (possibly corrupted) median value $\bm{\bar{\phi}}_{A_S}$ to the real (not corrupted) mean value of $\bm{\phi}_{A_S}$, denoted here with $\mu \in \mathbb{R}$.
The analysis  assumes that the evaluation returned by the subsample $A_S$ of the selected agents follows a Gaussian distribution with bounded variance $\sigma$. 
Let $\gamma < \nicefrac{1}{2}$ and $A_S = A_S^+ \cup A_S^-$ be the 
set of $m$ agents reporting an evaluation for the accuracy score of parameters $\theta_a$, where $A_S^+$ denotes the benign agents and $A_S^-$ the malicious agents. The analysis assumes that the points in $A_S^+$ are i.i.d.~from $\cN(\mu, \sigma^2)$ and $|A_S^-| < \gamma\,m$.

\begin{theorem}
\label{thm:trustability}
Let $t \geq \Phi^{-1}(\nicefrac{1}{2} + \beta)$, 
where 
$\Phi(t) = \Pr[X \leq t]$ 
is the cdf of the 
standard Normal distribution. Then the probability that the PT-DL median estimator $\bm{\bar{\phi}}_{A_S}$ differs from the real mean value $\mu$ is bounded by:
\begin{equation}
\label{eq:trustability}
\Pr\left[ | \bm{\bar{\phi}}_{A_S} - \mu | > t \sigma \right]
\leq
2 \exp\left(-2m\left( \Phi(t) - \nicefrac{\beta}{2}\right)^2\right).
\end{equation}
\end{theorem}

\begin{proof}
By scaling, and w.l.o.g.~assume $\sigma=1$ so that scores in $A_S^+$ are i.i.d.~from $\cN(\mu, 1)$ and the RHS of \eqref{eq:trustability} becomes $2 \exp\left(-2m\left( \Phi(t) - \nicefrac{1}{2}-\beta\right)^2\right).$ 
By assumption, $|A_S^+| > (1-\beta)m$ and thus the median of $A_S$ is at most the $(\frac{1}{2}+\beta)$-quantile of $A_S^+$, since $A_S^-$ contains only an $\beta$-fraction of all evaluations. It thus suffices to show that the $(\frac{1}{2}-\beta)$-quantile of $A_S^+$ is not too large. \\
For each $\phi_i \in A_S^+$, let $Y_i$ be the $\{0,1\}$-valued random variable which is $1$ if $X_i - \mu > t$ and $0$ otherwise. $Y_i$ follows an i.i.d.~Bernulli random process and 
$$
  \mathbb{E}[Y_i] = \Phi(-t) = 1 - \Phi(t).
$$
Moreover, the $(\frac{1}{2}-\beta)$-quantile of $A_S^+$ exceeds $\mu+t$ if $\frac{1}{2} \sum_{\phi_i \in A_S^+} Y_i > \frac{1}{2} - \beta$. Using a Chernoff bound, it follows that:
$$
  \Pr\left[\frac{1}{m}\sum_{\phi_i \in A_S^+} Y_i > 1 - \Phi(t) + s \right] \leq \exp(-2 m s^2).
$$
Since, $s >0$, by substituting $s$ with $\Phi(t)-\nicefrac{\beta}{2}$ gives the claim.
\end{proof}

The privacy analysis of PT-DL relies on the moment accountant for Sampled 
Gaussian (SG) mechanism \cite{mironov2019rnyi}, whose privacy is analyzed 
using \emph{R\'{e}nyi Differential Privacy} (RDP) \cite{Mironov_2017}. 

\begin{definition}
\label{def:SGM}
[Sampled Gaussian Mechanism]
Let $f:S\subseteq D \to \RR^d$ be a function mapping subsets $S$ of the input data $D$ to $\RR^d$. 
The Sampled Gaussian (SG) mechanism with sampling rate $0 < q \leq 1 $ and standard deviation $\sigma > 0$ is defined as follows:
{
\begin{align*}
SG_{q,\sigma}(D) \triangleq 
f(\{x: x \in D \  \mbox{is sampled with probability q}  \})
+ \cN(0, \sigma^2 \bm{I}),
\end{align*}
}
where each element of $D$ is sampled independently at random without replacement with probability $q$, and $\cN(0, \sigma^2 \bm{I})$ is the spherical $d$-dimensional Gaussian noise with per-coordinate variance $\sigma^2$.
\end{definition}

\begin{theorem}(($\alpha,\epsilon$)-RDP)
\label{thm:RDP}
A mechanism $f\!:\!\cD \!\to\! \cR$ with domain $\cD$ and range $\cR$ is said to have $\epsilon$-R\'{e}nyi differential privacy of order $\alpha$, or ($\alpha,\epsilon$)-RDP for short, if for any adjacent $D,D' \in \cD$ it holds that
\[
    \cD_\alpha(f(D) \parallel f(D')) \leq \epsilon,
\]
where 
$\cD_{\alpha} (P \parallel Q) \triangleq \frac{1}{1-\alpha} 
\log \EE_{x \sim Q} \left( \frac{P(x)}{Q(x)}\right)^{\alpha}$ is the R\'{e}nyi divergence of order $\alpha$ between two probability distributions $P$ and $Q$. 
\end{theorem}

The privacy analysis of the SG mechanism is described by the following Theorem from \citet{mironov2019rnyi}.
\begin{theorem}
\label{thm:RDP_sampling}
For a function $f\!:\! D \!\to\! \RR^d$, such that for all $D\sim D'$, $\|f(D) - f(D')\|_2 \leq 1$, the SG mechanism $SG_{q,\sigma}$ with sampling ratio $q$ and standard deviation $\sigma$ satisfies $(\alpha, \epsilon)$-RDP with:
{
\begin{subequations}
 \label{eq_app:RDP}
    \begin{align}
    \epsilon  & \leq \cD_{\alpha} 
        \left[ \cN (0, \sigma^2 ) \;\parallel\; (1-q) \cN (0, \sigma^2 ) + q \cN(1, \sigma^2 )  
        \right] \label{eq_app:RDPa}\\
    \epsilon & \leq \cD_{\alpha} 
        \left[ (1-q) \cN (0, \sigma^2 ) 
                + q \cN (1, \sigma^2 ) \;\parallel\; 
                \mathcal{N} (0, \sigma^2 ) 
        \right]. \label{eq_app:RDPb}
    \end{align}
\end{subequations}
}
\end{theorem}
The R\'{e}nyi divergences appearing in Equations \eqref{eq_app:RDPa}
and \eqref{eq_app:RDPb} can be computed numerically according to the
procedure described in \cite{mironov2019rnyi}. 
The final privacy loss in the $(\epsilon, \delta)$-differential
privacy model is obtained by observing that a mechanism satisfying
$(\alpha, \epsilon)$-R\'{e}nyi differential privacy also satisfies
$(\epsilon + \frac{\log \nicefrac{1}{\delta}}{\alpha-1},
\delta)$-differential privacy, for any $0< \delta < 1$
\cite{Mironov_2017}.

\begin{theorem}
\label{thm:privacy_loss}
Each agent update is differentially private. Additionally, denote $\epsilon_a, \ a \in [K]$ the per-round privacy loss of a PT-DL agent $a$. Regardless of the topology adopted, the per-round privacy loss of the PT-DL framework is: $\epsilon = \max_{a \in [K]} \epsilon_a$.
\end{theorem}
\begin{proof}
  For an agent $a \in [K]$, the differential privacy of the computations of the model parameters $\theta_a$ follows from the application of the SG mechanism in line (8) of Algorithm \ref{alg:agent}. 
  Next, since each agent holds a non-overlapping subsets of datasets $D$, by parallel composition of DP, the total privacy loss is $\max_{a \in [K]} \epsilon_a$.
\end{proof}

\begin{theorem}
\label{thm:acr_vs_privacy}
  Let PT-DL be executed for $N$ asynchronous communication rounds (ACRs). Then, the total privacy budget loss for topologies chain, tree, and star is, respectively: $\frac{N \epsilon}{K},\frac{N \epsilon}{\log (K)}$, and $N\epsilon.$
\end{theorem}
The result follows directly from the application of the sequential composition and observing that the PT-DL with chain topology requires $K$ ACRs to complete 1 round, tree requires $log(K)$ ACRs, and star requires a single ACR.

Theorem \ref{thm:acr_vs_privacy} shows that under a fixed communication cost constraint, PT-DL on a chain topology incurs less privacy loss than on a tree topology, and in turn less privacy loss than on a star topology.

\section{Experimental Analysis}
\label{sec:experimental_analysis}

\textbf{Datasets, Models, and Metrics} This section studies the behavior of the proposed PT-DL architectures on three classification tasks:
(1) \textit{MNIST} and (2) \textit{Fashion MNIST} comprises of 0 to 9 handwritten digits and articles images, respectively. These datasets have a training set of 60,000 examples and a test set of 10,000 examples. Each example is a 28x28 grayscale image, associated with a label from 10 classes. The task is to correctly classify the class associated with an image. 
(3) \textit{COVID-19 Chest X-Ray} is the first public COVID-19 CXR image data collection \cite{cohen2020covid,cohen2020covidProspective}, currently totaling 900 frontal chest X-ray images from over 26 countries. We combine it with Healthy Chest X-Ray images \cite{KERMANY20181122} to create a dataset containing, in total, 1800 data samples. The task is to correctly classify from an XRAY image whether a person is COVID-19 positive. 
 
The experiments consider a centralized neural network with 6 hidden layers as baseline with a total of 1,199,882 trainable parameters, and denoted SGD. Since this is a centralized model training on the whole dataset $D$, it is a representative of an upper bound on the task accuracy. 
The experiments further consider three PT-DL topologies, as described in Section \ref{sec:topologies}: a chain topology, denoted \chain, a binary tree, denoted \tree, and a star topology, denoted \star{} and compare them against FedAvg. When not otherwise specified, the experiments use 100 agents, each with 600 training and 100 testing examples. For the COVID-19 Chest X-Ray Dataset, each device has 7 samples of each class to train and 2 to test. All models are executed for $30$ runs (a.k.a., epochs) and the result report average and standard deviations of a 5-fold cross-validation.
A systematic model tuning for the proposed PT-DL was not investigated. 

The section first analyze the non-private versions of the models above an report their accuracy and scalability, their robustness to agents with highly biased datasets, to model inversion attacks, and to agent dropouts. Next, the section compares the differentially private version of the proposed models and of FedAvg to analyze the trade-off between privacy, accuracy, and communication costs.

\subsection{Accuracy and Scalability}
\begin{figure}[!t]
    \centering
    \includegraphics[width=\linewidth]{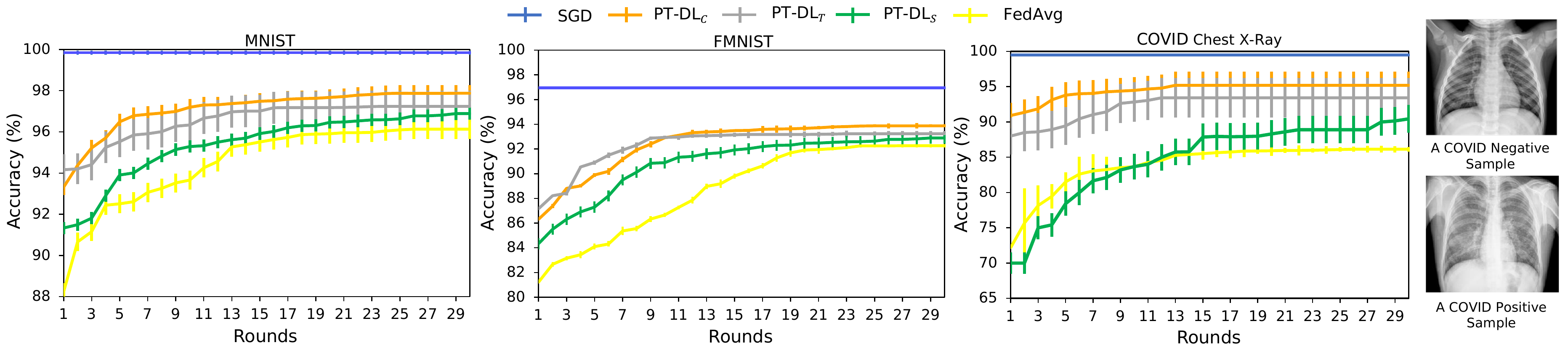}
    \caption{Algorithms accuracy per round on MNIST(left) and FMNIST(middle) and COVID-19 X-Ray(right) for $K=100$ agents.}
    \label{acc-results} 
\end{figure}

Figure~\ref{acc-results} reports a comparison of the accuracy of all models analyzed on the MNIST (right), FMNIST (center), and COVID-19 Chest X-Ray (left) datasets. 
Observe that the PT-DL models outperform FedAvg in all settings and datasets analyzed. The increase in accuracy is particularly notable for 
\chain{}, which performs consistently better than all other topologies analyzed. \tree{} obtains models that are typically dominated by \chain{}, but superior to \star{}. 
This result is explained by noticing that \star{} aggregates the parameters of all $K$ agents at each round. \tree, instead, aggregates only about $\nicefrac{K}{2}$ model parameters (i.e., those corresponding to the tree leaves). Finally, \chain{} does not perform any aggregation operation. By contrast, FedAvg aggregates the parameters of all $K$ agents multiple times per each round. 

\begin{figure}[!t]
    \centering
     \begin{subfigure}[c]{0.45\linewidth}
        \centering
         \includegraphics[width=\linewidth]{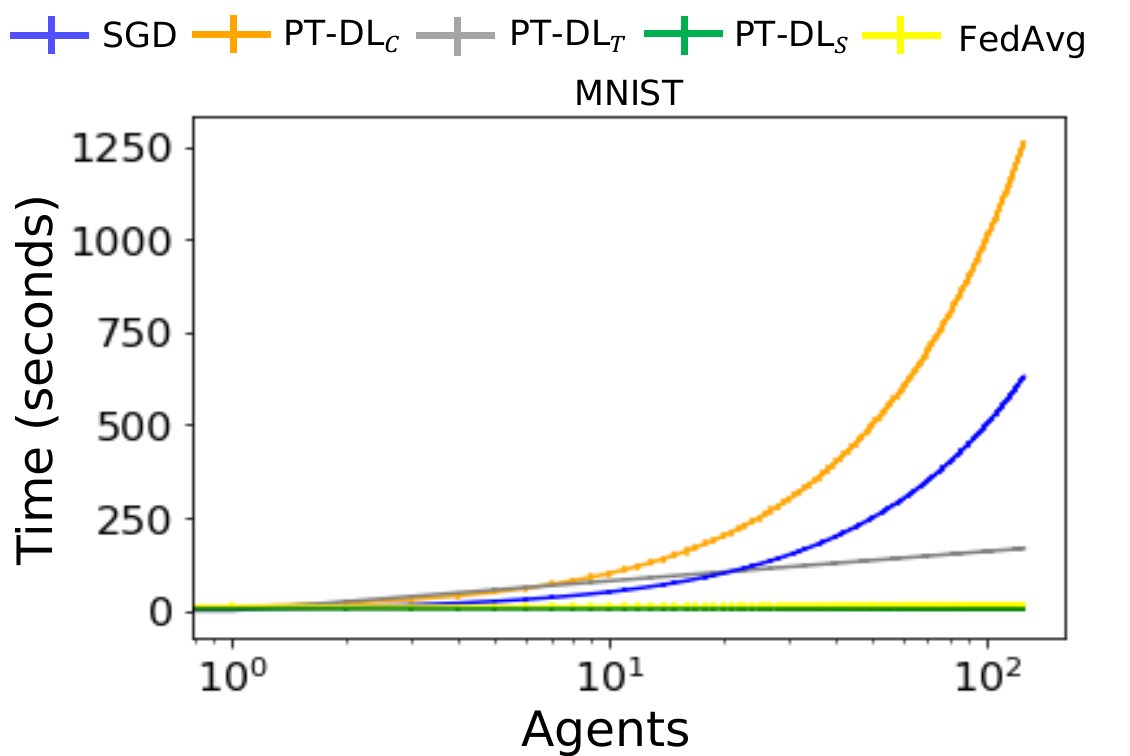}
         \caption{}
        \label{fig:time-bench}
     \end{subfigure}
     \begin{subtable}[c]{0.42\linewidth}
         \resizebox{\linewidth}{!}{%
         \begin{tabular}{l| l l l} 
         \toprule
         {Agents} & $10$ & $100$ & $1000$\\
         \midrule
         $\textrm{PT-DL}_C$ &  $ 0.102$ & $ 1.020$ & $ 10.20$\\ 
         $\textrm{PT-DL}_T$ &  $ 0.132$ & $ 1.510$ & $ 15.28$ \\
         $\textrm{PT-DL}_S$ &  $ 0.204$ & $ 2.040$ & $ 20.40$ \\ 
         FedAvg &              $ 0.600$ & $ 6.000$  & $ 60.00$ \\
         \bottomrule
        \end{tabular}
    }
    \caption{}
    \label{fig:network-bench}
    \end{subtable}
    \caption{Time (left) in seconds and network bandwidth (right) in GB, required to complete 1 round on MNIST data.
    }
    \label{fig:network-tests}
\end{figure}

Next, Figure~\ref{fig:time-bench} compares the time (in seconds) required by the algorithms to complete a single training round (i.e., it processes the updates of all the agents in the distributed network) at the increasing of the number of agents. The results are illustrated for the MNIST dataset, but they are consistent across all benchmarks.
The algorithms performance follows the following order, from faster to slower: \star, FedAvg, \tree, SGD, and \chain. 
The runtime is directly proportional with the level of parallelization brought by the different distributed topologies adopted. In particular, notice that both \tree, and \star, obtain an exponential speedup when compared with \chain and the centralized baseline. 
The result reveal a trade-off between computational runtime and accuracy of the PT-DL models, where \chain, the slowest protocol, is also the most accurate, while \star, the fastest protocol, is the least accurate. On the other hand, the \tree{} topology represents a good trade-off between accuracy and runtime. 
These observations demonstrate the practical benefits of the proposed framework.


Finally, Figure~\ref{fig:network-bench} compares the network bandwidth consumption among the algorithms analyzed at varying of the number of agents. The network bandwidth is measured in terms of total amount of data the agents send in Gigabytes. Once again, the results are illustrated for the MNIST dataset but are consistent throughout all datasets. 
The table illustrates a clear trend: The proposed PT-DL framework is able to compute models more efficiently than FedAvg as it requires less network bandwidth. Further, note that the bandwidth used is inversely proportional to the diameter of the communication graph induced by the PT-DL topologies, with \chain{} being the most efficient, followed by \tree, and finally \star.

\subsection{Robustness to Attacks and Biased Data}

\begin{figure}[tb] 
\centering
\includegraphics[width=0.8\linewidth]{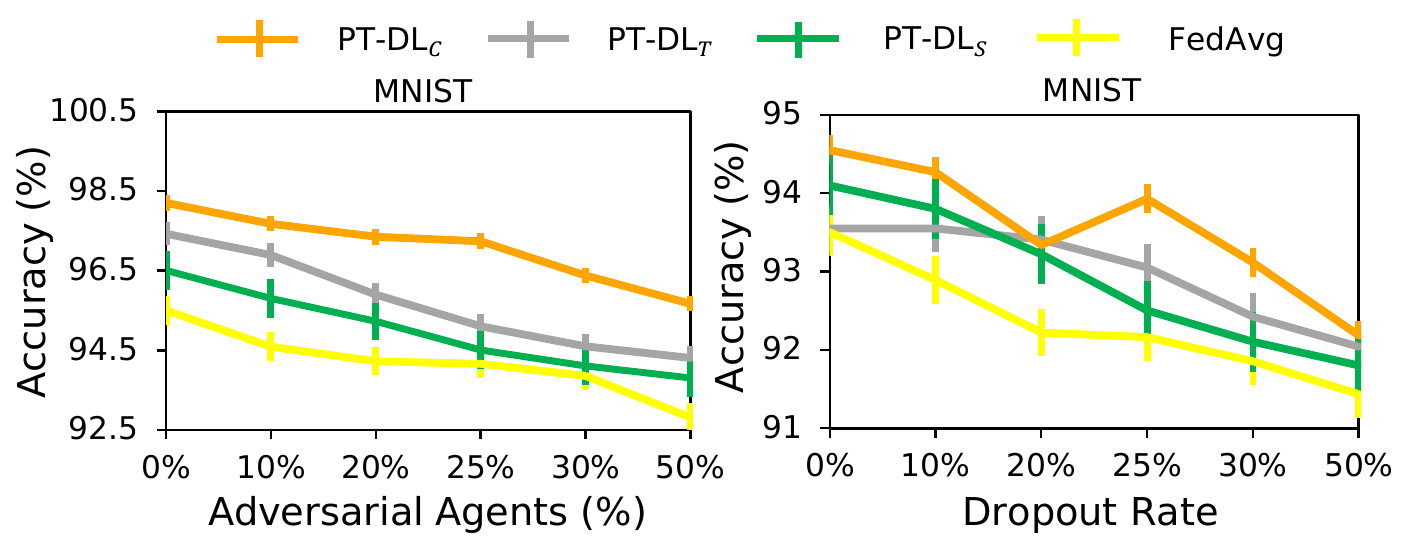}
\caption{Adversarial attacks (left) and dropout (right) resistance on
MNIST dataset and $K=100$ agents.} \label{fig-attack} 
\end{figure}

Next, the analysis focuses on the resilience of the decentralized algorithms
to adversarial attacks, agent dropouts, and to agents training over
highly biased data distributions.

In the first experiment, reported in Figure \ref{fig-attack} (left),
an increasing number of agents perform a model inversion attack
\cite{geiping2020inverting} prior sending their updates to the SCC or
to the centralized server (as in the case of FedAvg). For a fair
comparison, the experiments implement the same malicious agent detection scheme in FedAvg as that used by the proposed PT-DL with parameters $\kappa_1 = \kappa_2 = 0.05$, which rejects an update if the accuracy  discrepancy between the median evaluation and the reported one differs 
by more than 5\%. The experiments use $A_S = A$.
Observe that even with a significant portion (50\%) of the model
updates are being compromised, the global model trained by the
proposed framework sees only a small drop in accuracy. 

Next, Figure \ref{fig-attack} (right) reports the behavior of the
models as an increasing percentage of the agents drops out the
training process.  The trends are similar to those summarized above,
where only a small drop in accuracy is observed, even when half of the
agents drop out.  Once again, the results illustrate that the proposed
PT-DL framework outperforms the distributed baseline, FedAvg.
\emph{These trends are a clear indication towards the ability of the
proposed protocol to mitigate adversarial attacks as well as agents
drop out.}

\begin{figure}[tb]
\centering
\includegraphics[width=0.8\linewidth]{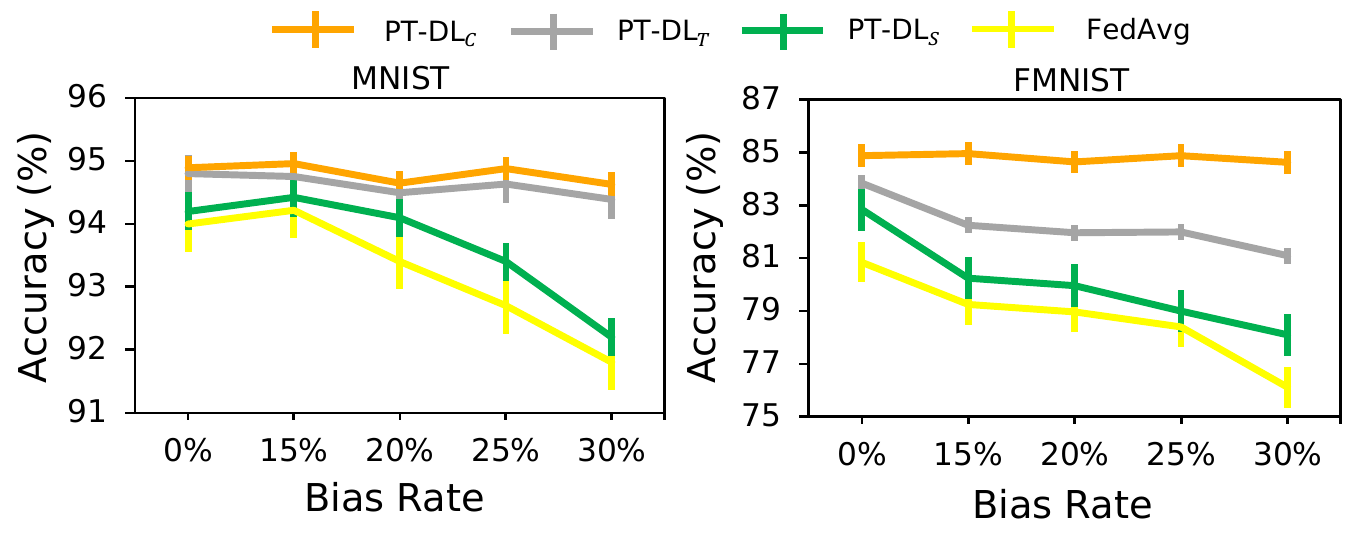}
\caption{Algorithms accuracy on MNIST(left) and FMNIST(right) dataset 
for varying bias rates. $K=100$ agents}
\label{fig-mnist-bias}
\end{figure}

\begin{figure*}[t]
\centering
\includegraphics[width=.99\textwidth]{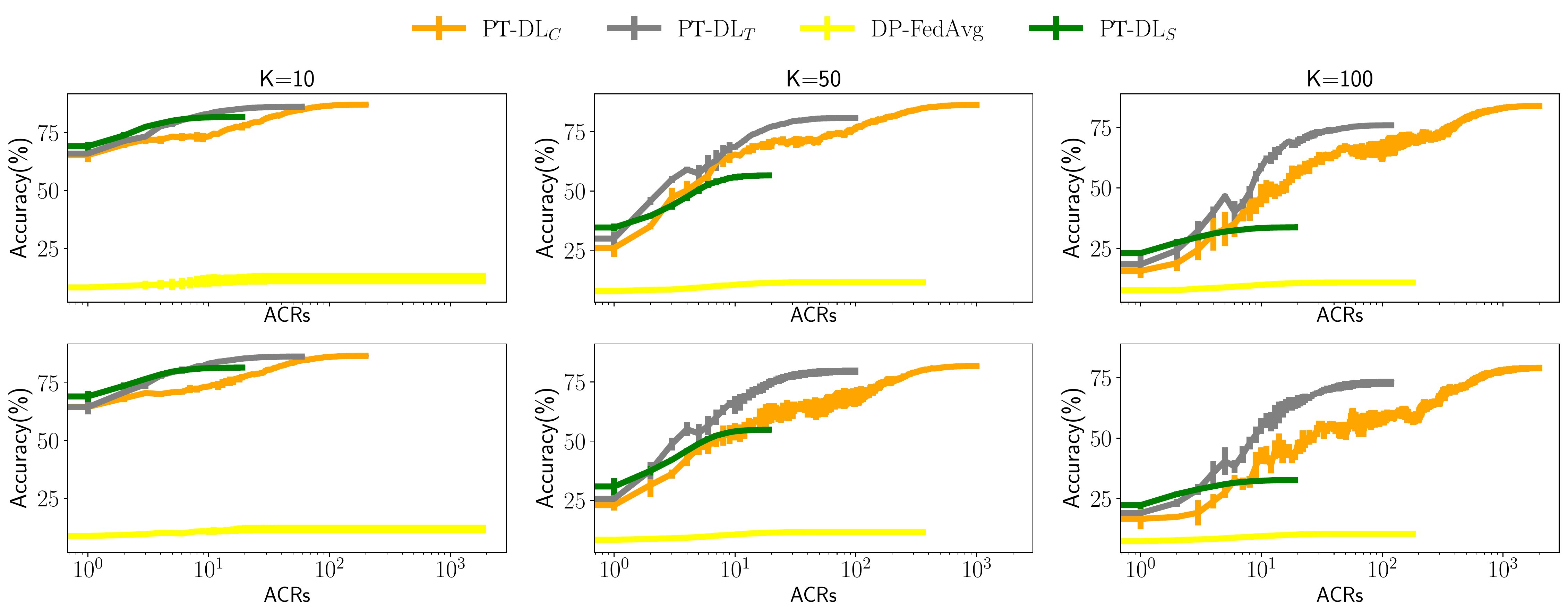}
\caption{Accuracy vs ACRs on \textbf{unbiased} MNIST (top) and \textbf{biased}  MNIST (bottom) datasets for $K=10$ (top), $K=50$ (middle), and $K=100$ (bottom) agents. The final privacy losses for each model with $K=10$, $K=50$, and $K=50$ respectively are $0.5, 1.1$ and $1.6$.}
\label{fig-mnist-dp}
\end{figure*}

Finally, the section analyses the results in a highly biased setting using the MNIST and FMNIST datasets. In the experiments each agent is given a dataset that contains a certain percentage of one particular label, while the rest of the data represents uniformly data with the remaining labels. 
This is useful to test agents with biased datasets influenced by a majority class. The results are summarized in Figure \ref{fig-mnist-bias}. 
Firstly, observe that all the PT-DL versions dominate FedAvg in terms of accuracy. Next, observe that the drop in accuracy is almost imperceptible for \chain{} and \tree, while PT-DL on the star topology reports the largest drops, at the increasing of the distribution bias. 
\emph{The results show that the proposed framework performs reliably in real-world conditions where agent data distribution may vary widely}.

\subsection{Privacy/Accuracy trade-off} 
\label{sec:privacy}
Finally, the analysis focuses on the accuracy of the distributed models under the differential privacy constraints. 
It follows similar settings as those described above and, additionally, it focuses on evaluating the models quality on two additional dimensions: varying of the number of agents $K$, and the number of Asynchronous Communication Rounds (ACRs) (See Definition \ref{def2}). 
To ensure privacy, this section uses DP-FedAvg \cite{mcmahan2017communication} (in lieu of FedAvg), under the privacy model adopted in the paper. 
The privacy settings for all models are: $c = 10$ and $\sigma = 2.0$ (line 8 of Algorithm \ref{alg:agent}) and the probability of pure DP violation $\delta = 1e-3$. The total privacy loss is computed based on the moment accountant method, which was discussed in Section \ref{sec:privacy_analysis}.

The first experiment evaluates all models using (unbiased) MNIST (Figure \ref{fig-mnist-dp} (top)) using 10 (left), 50 (center), and 100 (right) agents. The results for other datasets follow the same trends and are reported in the supplemental material.
The experiment fixes the maximal number of rounds each algorithm makes to 20, and reports, on the x-axis, the number of ACRs performed by the algorithms, which can also be interpreted as the number of parallel communications the agents perform during the evolution of the distributed process. A discussion on the relation of the ACRs and the number of updates illustrated in the figure plots is given in the appendix. 
The resulting final privacy losses for all algorithms are $\epsilon=0.5$ (for K=10), $1.1$ (for K=50), and $1.6$ (for K=100). The privacy loss increases (sub-linearly) with the number of agents as the dataset size $n_a$ each agent is given decreases while the mini-batch size $b$ is fixed.  


Firstly, observe that PT-DL produces private models that are significantly more accurate to those produced by DP-FedAvg, under the same privacy constraints. Our analysis indicates that this is due to the different number of aggregation operations performed by the various algorithms. This crucial behavior was also observed in \citet{mcmahan2017communication} that noted that for general non-convex objectives, averaging model parameters could produce an arbitrarily bad model.
Under a very tight privacy constraint $\epsilon \approx 1$, \chain{} and \tree{} consistently achieve models with more than 80\% and 76\% accuracy, respectively, while models produced by \star{} degrade their performance as the number of agents increase. This is due because experiments with small (large) $K$ have agents holding larger (smaller) datasets: $n_a = n/K$ for each agent $a \in [K]$. Since the privacy loss is affected by the data set size, for fixed mini-batch size $b$ and sampling rate $q$ (See Theorem \ref{thm:RDP_sampling}), models with lager $K$ will result in more noisy updates. This is however, greatly mitigated by the chain and tree topologies of PT-DL that perform far less aggregation operations, when compared to \star{} or DP-FedAvg.

The second experiment, illustrated in Figure \ref{fig-mnist-dp} (bottom),
focuses on biased datasets (bias rate of 30\%). The figure reports results for MNIST, but the conclusions extend to the other datasets adopted (see appendix). The results follow the same trends as those outlined above. Additionally, notice that the proposed models experience only a negligible drop in accuracy. 

\emph{These experiments demonstrate the robustness of the proposed models under a non-IID setting, a varying number of agents, and strict privacy constraints.}

\section{Conclusions}
\label{sec:conclusions}
This paper presented a privacy-preserving and trustable distributed learning (PT-DL) framework for multi-agent learning. 
The proposed framework is fully decentralized and relies on the notion of Differential Privacy to protect the privacy of the agents data and Ethereum smart contracts to ensure trustability. 
The paper showed the trustable scheme implemented in PT-DL is robust, with high probability, up to the case where half of the agents may collude.
The experimental evaluation illustrates the benefits of the proposed model in terms of accuracy and scalability, robustness to agents with highly biased datasets, to model inversion attacks, and to agent dropouts. Additionally, the model is shown to outperform standard differentially private federated learning algorithms in terms of accuracy and communication costs, for fixed privacy constraints. The results show that PT-DL may be a step toward a practical tool for privacy-preserving and trustable multi-agent learning.



\bibliographystyle{abbrvnat} 
\bibliography{bibliography}

\newpage
\appendix
\section{Additional Background on Blockchain}

\subsection*{\textbf EVM} The Ethereum Virtual Machine (EVM) is a distributed Turing complete system that allows us to build Smart Contracts written in Solidity or Vyper, and execute them using the Ethereum Blockchain.

A temporary Ethereum address generated to sign transactions and to communicate with the AP-DL Smart Contract. This ensures that the Smart Contract does not use the same Ethereum address to refer to the device only once.

\subsection*{\textbf Zero-Knowledge Proof}  A Zero-Knowledge Proof is a method by which one participant in the network \textit{(prover)} can prove to another participant \textit{(verifier)} that they know a value x, without revealing the value of x. There are two types of ZKPs — interactive and non-interactive. An interactive proof involves a series of questions from the \textit{verifier} to the \textit{prover} to prove their knowledge whereas a non-interactive proof relies on the \textit{verifier} picking a random challenge for the \textit{prover} to solve. Multiple interactions between the \textit{prover} and \textit{verifier} become unnecessary as the proof exists in a single message sent from the \textit{prover} to the \textit{verifier}. We utilize a non-interactive ZK proof also known as ZK-STARK.
The Zero Knowledge proof allows the progress made by the agent to be verified without revealing any data. By storing the proof on-chain agents can verify the validity of the progress of any other agent without requiring any interaction. Secondly, as anyone who has the private key to the signing account of the ZKP can prove that this progress was indeed done by them, they can be rewarded or penalized depending on the update they sent. So in an incentivized system, agents would be motivated to publish good updates.

\subsection*{\textbf Blockchain Operations} Connecting to the blockchain and sending transactions to it requires access to an Ethereum Node. There are two ways of accessing one. First, running a node locally on a system that directly connects to the network, downloading it's entire state. Secondly, by using a Remote RPC Endpoint which is essentially a node running on a remote machine that handles running the node and relays our transaction data. The transaction data cannot be manipulated as it is signed by a private key. This allows IoT Devices to data centers to interact with blockchains. Permissionless Blockchains like Ethereum allow anyone to create decentralized applications using Smart Contracts. These smart contracts are written in programming languages like Solidity and Vyper. All this combined creates a decentralized and turing complete atomic computation environment.

\section{Incentivization}
One of abilities of blockchain is to support decentralized finance. This allows us to incentivize the training process rewarding and penalizing agents relative to their contribution towards improving the accuracy of the global model. After an agent completes its training process it creates a ZKP of its progress. It internally uses a Commit Reveal Scheme. The agent will locally generate a random number(the reveal) with a salt and then hash it and send it on-chain (the commit). Finally, we’ll hash their random number (that the miner shouldn't know about) with the block-hash on the commit block (that the player couldn't know about). This final hash is a pretty good source of randomness on-chain because it gives the player an assurance that the miner didn’t manipulate it. This process is used to hide the identity of the agent if they choose to redeem their rewards using the reveal to a new address with no transaction history. If the model update from this agent is accepted, then the ZKP is stored on the chain with the users address and can be redeemed by any address as long as the have the reveal to value of the commit. The ZKP also allows another agent to validate the progress of an other agent using the ZKP without requiring to interact with the agent. These measuring help in incentivizing the users to push good updates and make agents accountable to the other agents in a privacy-preserving manner.

\section{Software and Models Settings}

\smallskip\noindent\textbf{Computing Infrastructure} 
All experiments were performed on a cluster equipped with Intel(R) 
Xeon(R) Platinum 8260 CPU @ 2.40GHz and 8GB of RAM.

\smallskip\noindent\textbf{Software and Libraries}  
All models and experiments were written in Python 3.7. 
All models in the paper were implemented in Pytorch 1.5.0. The Tensorflow Privacy library \href{https://github.com/tensorflow/privacy} is utilized for privacy computation. 

\smallskip\noindent\textbf{Architectures}  The network architecture for MNIST/FMNIST dataset is reported in Table \ref{table-struct}. For COVID dataset, the RESNET18 model was utilized due to its superior performance in classifying  normal and COVID-19 pneumonia classes, \cite{Chowdhury_2020}. To speed up the training progress, all layers except the last fully connected one were frozen. Hence, the parameter learning takes place only at this last layer. 

\begin{table}[!tbh]
\centering
{%
\begin{tabular}{l l l l} 
 \toprule
 Layer (type) & Output Shape & Param \# & Tr.~Param \# \\
 \midrule
 Conv2d-1 & $ [1, 32, 26, 26] $ & $ 320$ & $ 320$  \\ 
 Conv2d-2 &  $ [1, 64, 24, 24]$ & $ 18,496$ & $ 18,496$ \\
 Dropout2d-3 & $ [1, 64, 12, 12] $ & $ 0$ & $ 0$ \\ 
 Linear-4 & $ [1, 128]$ & $ 1,179,776 $ & $ 1,179,776 $  \\
 Dropout2d-5 & $ [1, 128] $ & $ 0 $ & $ 0 $ \\
 Linear-6 & $ [1, 10]$ & $ 1,290$ & $ 1,290$ \\
 \bottomrule
\end{tabular}
}
\caption{Model Structure}
\label{table-struct}
\end{table}

\smallskip\noindent\textbf{Code} 
The implementation of all models and the relevant experiments  will be released upon publication.

\smallskip\noindent\textbf{Algorithms' Setting} 
The parameters settings for all models across datasets are given in Table  \ref{tab:parameter_setting_1}.  The settings for private extension of these models are summarized in Table  \ref{tab:parameter_setting_2}.

\begin{itemize}
    \item Inner epochs: the number of local epochs each agent performs on its own data in a single round.  
    \item ACRs: the total number of asynchronous communication rounds(ACRs, see Definition 2) that each agent perform during the course of model's training. This parameter depends strongly on model, the number of agents $K$ and the total number of training sample $n$ in case of Fed-Avg. Given a fixed number of rounds $N$,  with $K$ agents, the number of ACRs for chain, tree, star topology respectively is: $N K $, $N \log_2 K $ and $N$. 
\end{itemize}

\begin{table}[htb]
\centering
\resizebox{0.86\linewidth}{!}{%
\begin{tabular}{r r r r r r}
\toprule
 Model &  Data &  Batch size & \# Rounds & \# Inner epochs & \# ACRs \\
\midrule
 SGD &MNIST/FMNIST & 64 & NA &30  & NA \\
 $\textrm{PA-DL}_C$ & MNIST/FMNIST & 64 &30  & 1 &  $30 \  K$  \\
  $\textrm{PA-DL}_T$ & MNIST/FMNIST&  64 &30  & 1 & $30\  \log_2 K $    \\
   $\textrm{PA-DL}_S$ & MNIST/FMNIST&  64 &30  & 1 &  30\\ 
   $\textrm{FedAvg}$ & MNIST/FMNIST&  64 &30  & NA & $\frac{30  n}{K 64 }$   \\
   \midrule
       SGD & COVID &  8& NA & 30 & NA \\
 $\textrm{PA-DL}_C$ & COVID & 8 &30  & 1 &  $30\  K$ \\
  $\textrm{PA-DL}_T$ & COVID&  8 &30  & 1 & $30\  \log_2 K $   \\
   $\textrm{PA-DL}_S$ &  COVID&  8 &30  & 1 &  30 \\
      $\textrm{FedAvg}$ & COVID&  8 &30  & NA & $\frac{30  n}{K 16 }$   \\
\bottomrule
\end{tabular}
}
\caption{Parameters settings for all \textbf{non-private} models presented in the paper, where $K$ is the number of agents, $n$ is the total number of training samples.  }
\label{tab:parameter_setting_1}
\end{table}

\begin{table}[htb]
\centering
\resizebox{0.85\linewidth}{!}{%
\begin{tabular}{r r r r r r}
\toprule
 Model &  Data &  Batch size & \# Rounds & \# Inner epochs & \# ACRs \\
\midrule
 $\textrm{PA-DL}_C$ & MNIST/FMNIST & 64 &20  & 1 &  $20 \  K$  \\
  $\textrm{PA-DL}_T$ & MNIST/FMNIST&  64 &20  & 1 & $20\  \log_2 K $    \\
   $\textrm{PA-DL}_S$ & MNIST/FMNIST&  64 &20  & 1 &  20\\ 
   $\textrm{DP-FedAvg}$ & MNIST/FMNIST&  64 &20  & NA & $\frac{20  n}{K 64 }$   \\
   \midrule
 $\textrm{PA-DL}_C$ & COVID & 16 &20  & 1 &  $20\  K$ \\
  $\textrm{PA-DL}_T$ & COVID&  16 &20  & 1 & $20\  \log_2 K $   \\
   $\textrm{PA-DL}_S$ &  COVID&  16 &20  & 1 &  20 \\
      $\textrm{DP-FedAvg}$ & COVID&  16 &20  & NA & $\frac{20  n}{K 16 }$   \\
\bottomrule
\end{tabular}
}
\caption{Parameters settings for private extensions of all models presented in the paper, where $K$ is the number of agents, $n$ is the total number of training samples.  }
\label{tab:parameter_setting_2}
\end{table}

\section{Extended Results}

The experiments reported below extend the analysis reported in Section 7.3 of the main paper, that focuses on the impact of the privacy constraints onto the accuracy of the different models. Figure \ref{fig:fig-dp-2} provides a comparison across all models on the unbiased versions of MNIST (left) and FMNIST (right) datasets, while Figure \ref{fig-biased-dp-2} illustrates the performance of all models on the biased counterparts of these datasets. The parameter settings for all models are summarized in Table  \ref{tab:parameter_setting_2}. The resulting final privacy losses for all algorithms are $\epsilon=0.5$ (for K=10), $1.1$ (for K=50), and $1.6$ (for K=100). The privacy loss increases (sub-linearly) with the number of agents as the dataset size $n_a$ each agent is given decreases while the mini-batch size $b$ is fixed.  

Notice that the reason why the algorithms stop to different ACR is because they are all ran for 20 rounds, and their total number of ACRs is reported in Table \ref{tab:parameter_setting_2}.

The results follow trends as those reported in the main paper. 
In particular, PA-DL produces private models that are significantly more accurate to those produced by DP-FedAvg, under the same privacy constraints. 

Similar observation are also made for the COVID dataset, illustrated in Figure \ref{fig-dp-COVID-19-2}. Since this is a very small dataset, the experiments use $K=5$ and $K=10$ agents, and the data samples are distributed uniformly across the $K$ agents.

\begin{figure}[t]
\centering
\includegraphics[width=0.85\textwidth]{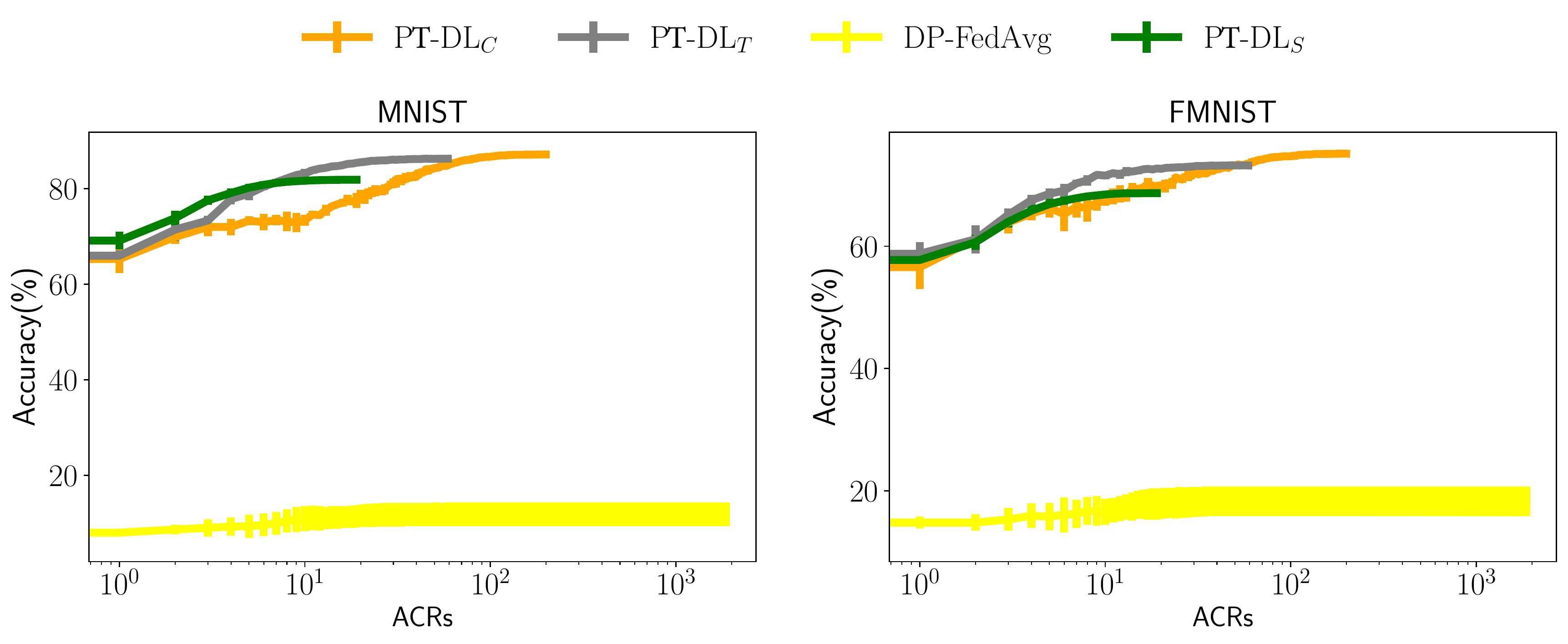} 
\includegraphics[width=0.85\textwidth]{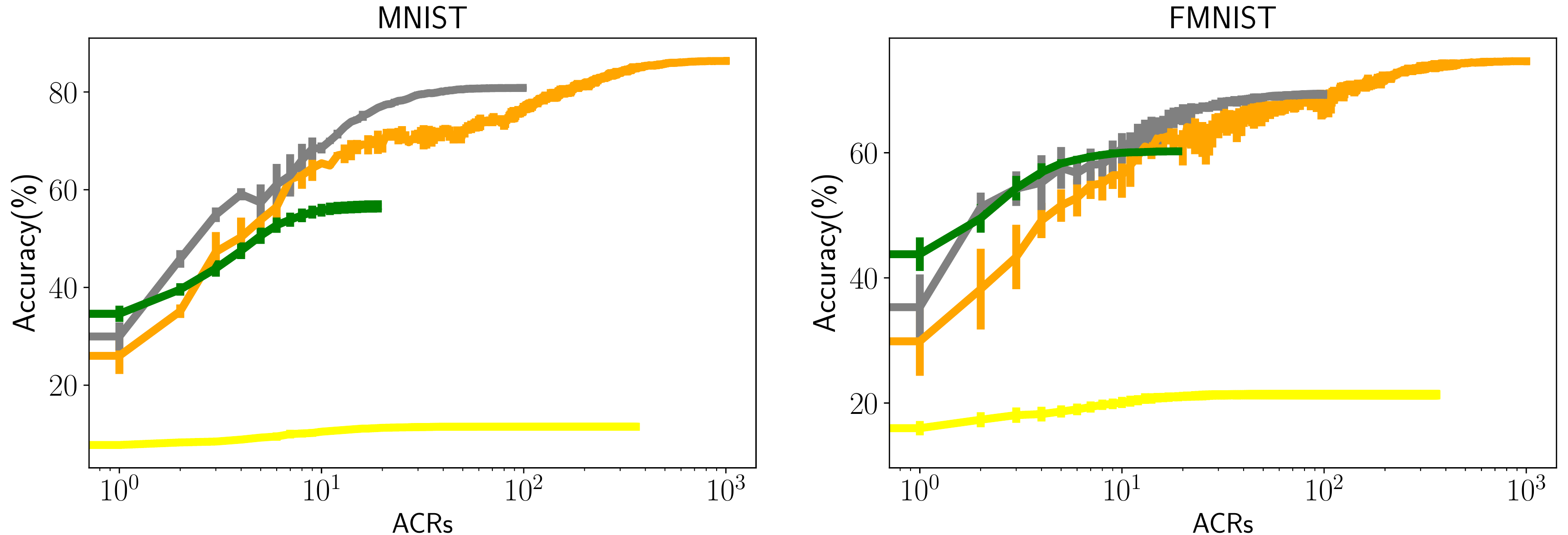} 
\includegraphics[width=0.85\textwidth]{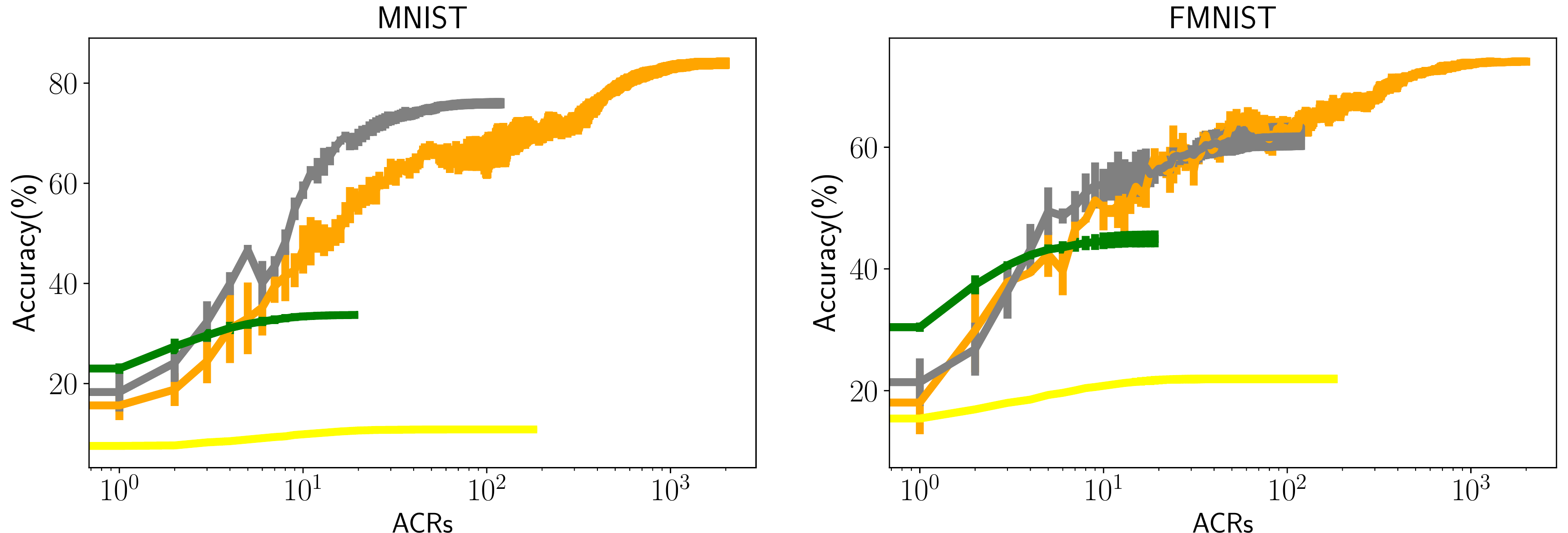} 
\caption{Accuracy vs ACRs on unbiased MNIST (left) and unbiased  FMNIST (right) datasets for $K=10$ (top), $K=50$ (middle), and $K=100$ (bottom) agents.  The final privacy loss of each model for $K=10$, $K=50$, and $K=50$ respectively are $0.5, 1.1$ and $1.6$.
\label{fig:fig-dp-2}}
\end{figure}

\begin{figure}[t]
\centering
\includegraphics[width=0.85\textwidth]{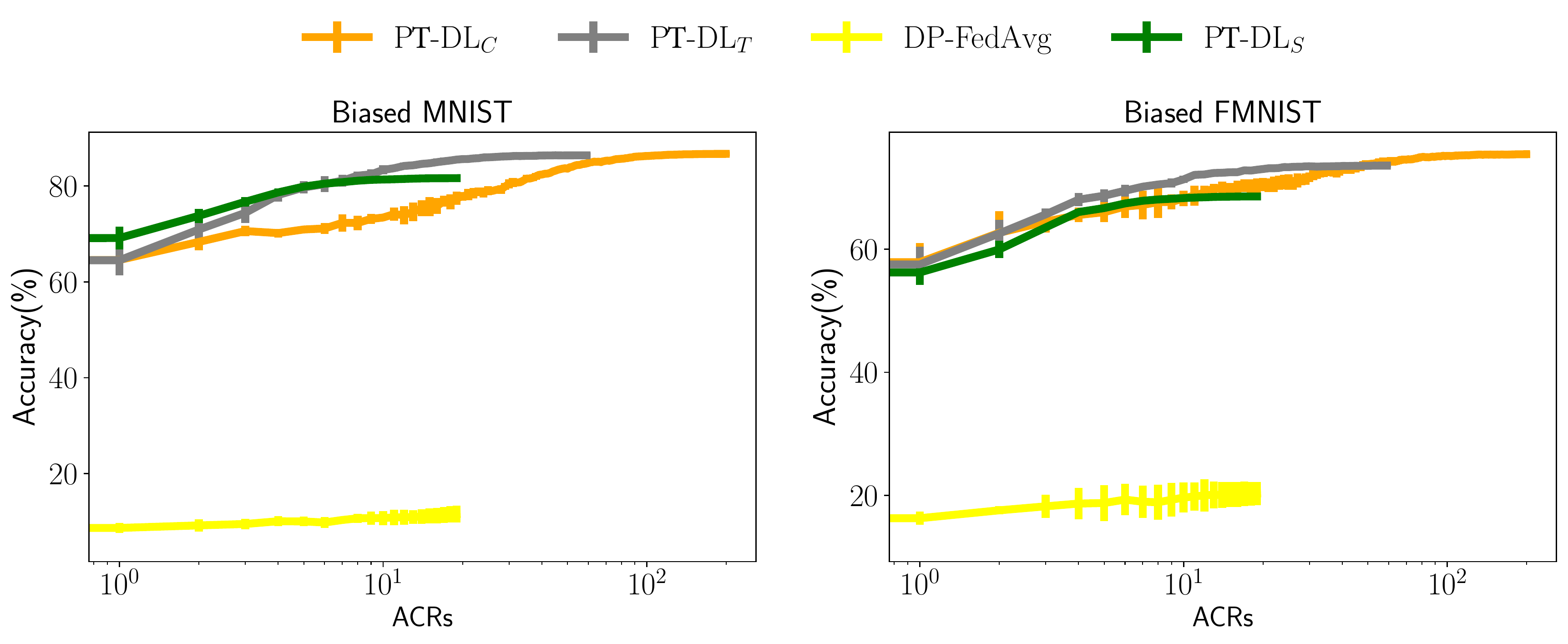}
\includegraphics[width=0.85\textwidth]{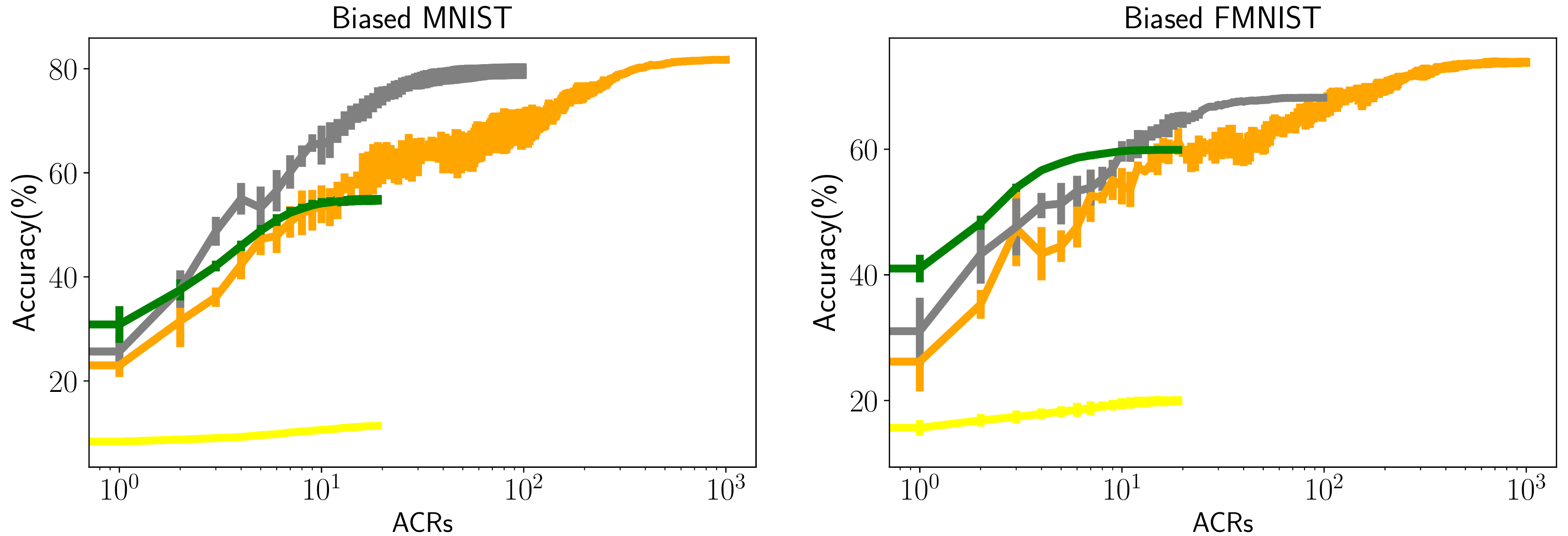}
\includegraphics[width=0.85\textwidth]{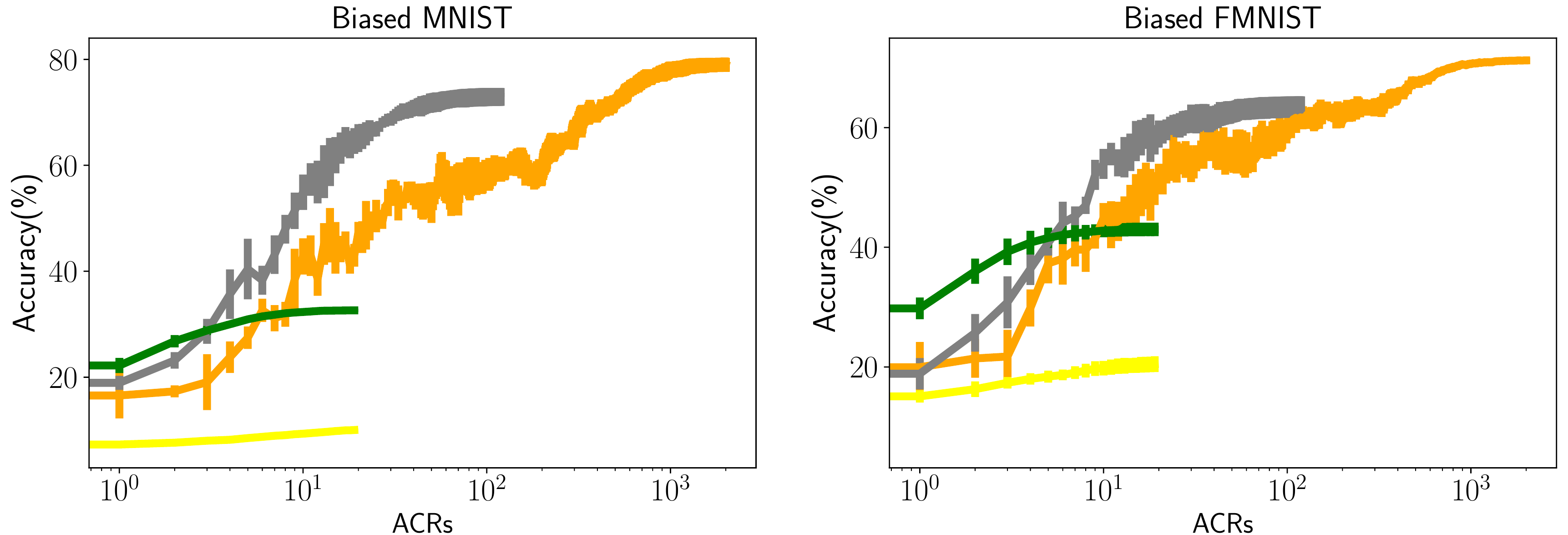} 
\caption{Accuracy vs ACRs on \textbf{biased} MNIST (left) and \textbf{biased}  FMNIST (right) datasets for $K=10$ (top), $K=50$ (middle), and $K=100$ (bottom) agents. The final privacy loss of each model for $K=10$, $K=50$, and $K=50$ respectively are $0.5, 1.1$ and $1.6$. }
\label{fig-biased-dp-2}
\end{figure}

\begin{figure}[t]
\centering
\includegraphics[width=0.85\textwidth]{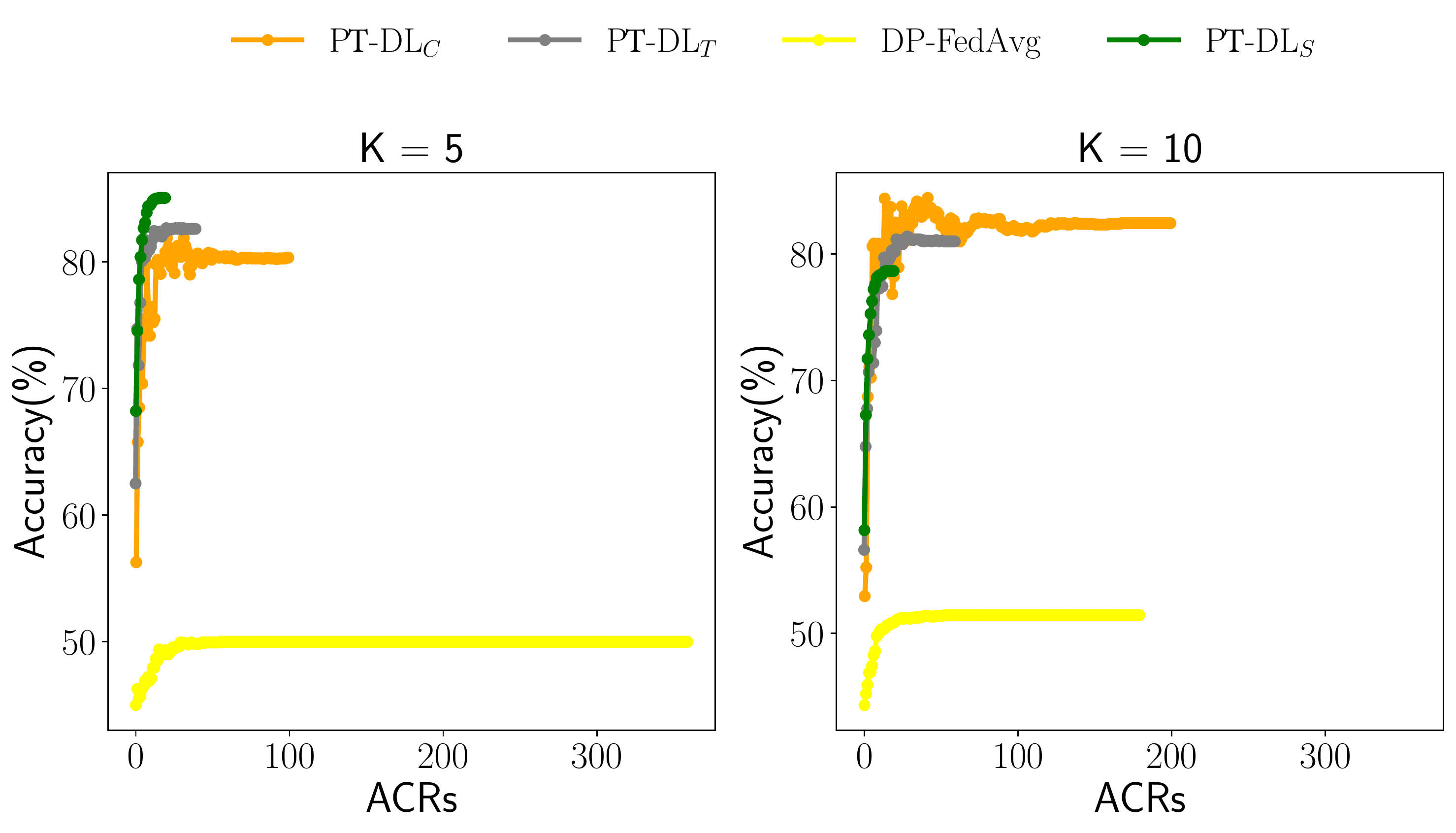} 
\caption{Accuracy vs ACRs on COVID-19 dataset for $K=5$ (left) and $K=10$(right). The final privacy loss of each model for $K=5$ and $K=10$ respectively are $2.4$ and $3.5$. }
\label{fig-dp-COVID-19-2}
\end{figure}

\end{document}